\documentclass[10pt,a4paper]{article}

\usepackage{a4wide}
\usepackage{amsmath}
\usepackage{amssymb}
\usepackage{amsthm}

\usepackage{graphicx}
\usepackage{subfigure}

\renewcommand{\vec}[1]{{\mathbf #1}}
\providecommand{\defeq}[0]{\stackrel{\text{def}}{=}}
\providecommand{\persp}[1]{#1_\oslash}

\newtheorem{theorem}{Theorem}

\newtheorem{lemma}{Lemma}

\theoremstyle{remark}
\newtheorem{remark}{Remark}

\date{August 15, 2013}

\begin{document}
\title{Compact Relaxations for MAP Inference in Pairwise MRFs with Piecewise Linear Priors}

\author{Christopher~Zach\thanks{Microsoft Research Cambridge} \and and~Christian~H\"{a}ne\thanks{ETH Z\"urich, Switzerland}}


\maketitle

\begin{abstract}
Label assignment problems with large state spaces are important tasks
especially in computer vision. Often the pairwise interaction (or smoothness
prior) between labels assigned at adjacent nodes (or pixels) can be described
as a function of the label difference. Exact inference in such labeling tasks
is still difficult, and therefore approximate inference methods based on a
linear programming (LP) relaxation are commonly used in practice. In this work
we study how compact linear programs can be constructed for general piecwise
linear smoothness priors. The number of unknowns is $O(LK$) per pairwise
clique in terms of the state space size $L$ and the number of linear segments
$K$. This compares to an $O(L^2)$ size complexity of the standard LP
relaxation if the piecewise linear structure is ignored. Our compact
construction and the standard LP relaxation are equivalent and lead to the
same (approximate) label assignment.
\end{abstract}

\section{Introduction}

Determining a maximum a-posteriori (MAP) solution in graphical models over
discrete states, or equivalently finding a minimizer of a corresponding
energy, is one of the fundamental tools in machine learning and computer
vision. In this work we focus on problems with at most pairwise cliques (and
associated pairwise potentials) in the graphical model. In the following we
use the terms ``pairwise potential'' and ``(smoothness) prior''
synonymously. Since exact inference in graphical models (even with at most
pairwise potentials) is generally not tractable, research has been focused on
tractable and high-quality \emph{approximate} inference algorithms, of which
the linear programming (LP) relaxation for discrete inference tasks
(e.g.~\cite{chekuri2004linear,werner2007maxsum_review,wainwright2008graphical,sontag2011introduction})
has received much attention. In some cases the LP relaxation solves the
inference problem
exactly~\cite{ishikawa2003exact,kolmogorob2005optimality,schlesinger2007permuted},
i.e.\ the relaxation is tight for certain problem classes. Due to the specific
structure of the resulting linear program generic methods to solve linear
programs are inefficient and therefore many specialized algorithms to find a
minimizer have been proposed in the literature. Since the primal linear
program for MAP estimation in pairwise problems has a quadratic number of
unknowns in terms of the state space size $L$ (per edge in the underlying
graph), the dual program with only a linear number of unknowns (but with a
quadratic number of e.g.\ constraints) is more appealing. Belief propagation
inspired message passing methods optimize this dual program in a
block-coordinate
schedule~\cite{kovalesky75msd,kolmogorov2006convergent,globerson2007fixing,hazan2010normproduct}.
As block-coordinate methods applied on a non-smooth and not strictly concave
problem these approaches iteratively increase the dual objective but are not
guaranteed to find a global maximizer of the (concave dual) problem. This is
well known in the literature, and we validate this occasional ``early
stopping'' behavior in the experimental section. Supergradient methods with an
appropriate stepsize rule are guaranteed to converge to a maximizer, but have
a slow $O(1/\sqrt{T})$ convergence rate (where $T$ is the iteration count). To
our knowledge all variations of faster $O(1/T)$ proximal methods explicitly or
implicitly maintain $O(L^2)$ primal unknowns, and are therefore prohibitively
expensive (in terms of memory requirements) for large state spaces. Variable
smoothing methods for non-smooth problems (e.g.~\cite{bot2012variable}) have
an $O(\ln(T)/T)$ convergence rate and require only $O(L)$ unknowns if applied
on the dual problem, but in practice show slow convergence in our experience
(as verified in our experimental section).

\begin{figure}[h]
  \centering
  \subfigure[]{ \includegraphics[width=0.23\linewidth]{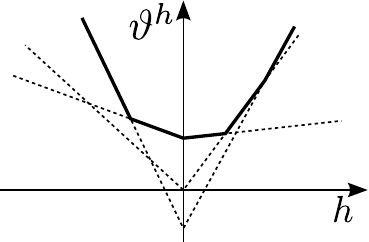}}
  \subfigure[]{\includegraphics[width=0.23\linewidth]{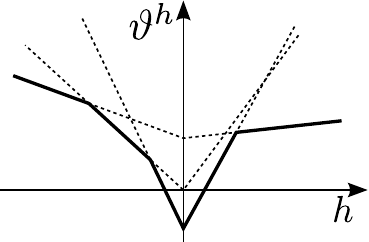}}
  \subfigure[]{ \includegraphics[width=0.23\linewidth]{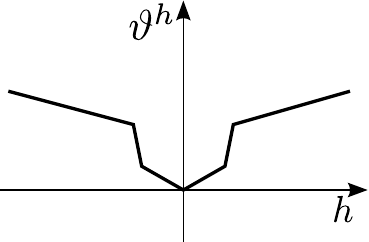}}
  \subfigure[]{\includegraphics[width=0.23\linewidth]{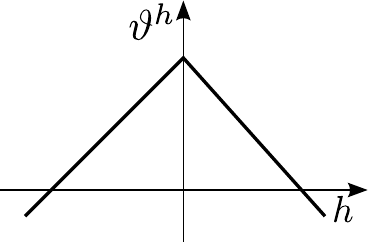}}
  \caption{Pairwise potentials (in terms of the label difference $h$) suitable
    for compact relaxations. }
  \label{fig:maxmin}
\end{figure}

A natural question is whether the size of the primal program can be reduced if
the pairwise potentials are not arbitrary but have some useful ``structure.''
For particular pairwise potentials such as the $L^1$-norm of label
differences~\cite{ishikawa1998segmentation,boykov98markov} or truncated $L^1$
priors~\cite{zach2013connecting} the answer is affirmative. We generalize
these particular results to pairwise potentials that are piecewise linear in
terms of the label difference. In many computer vision related inference
problems the state space is a set of numeric values and the employed
smoothness prior is naturally based on label value differences.  Some of the
potentials addressed in our work are illustrated in Fig.~\ref{fig:maxmin}. We
show that for such pairwise potentials (but with arbitrary unary ones) we can
reformulate the primal linear programs in order to reduce the number of primal
variables from $O(L^2)$ to $O(KL)$, where $K$ is the number of linear pieces
defining the pairwise potential. The number of dual unknowns (or primal
constraints, respectively) increases to $O(KL)$, meaning an overall $O(KL)$
problem size. Specifically, we consider pairwise potentials that can be
written as pointwise minimum of convex ones.  Since our construction is a
reformulation, there is a correspondence between minimizers of the original LP
and the ones from our reduced program. Thus, our construction does not
weaken the LP relaxation for inference.

This manuscript is organized as follows: after introducing relevant notations
and briefly reviewing approximate MAP inference in
Section~\ref{sec:background}, we state our main result in
Section~\ref{sec:main_result}. The corresponding proof is constructive and the
two main ingredients are presented in Sections~\ref{sec:conv} (convex pairwise
priors) and~\ref{sec:metr} (pointwise minimum of pairwise potentials),
respectively. Since the underlying techniques are useful in their own right,
we provide the material in separate (and relatively self-contained) sections.
In Section~\ref{sec:isotropic} we discuss extensions of the proposed
reformulations to enable more isotropic behavior of solutions, which can be
relevant in image processing applications. In Section~\ref{sec:experiment} we
experimentally verify that message passing methods can stop early, and we
demonstrate our approach in an image denoising experiment.

\section{Background}
\label{sec:background}

In this section we introduce some notation used throughout the manuscript, and
further provide a short review on approximate inference for labeling problems.

\subsection{Notations}

The domain of the considered label assignment task is a graph ${\cal G} =
({\cal V}, {\cal E})$ with node set $\cal V$ and edge set $\cal E$. In
computer vision and image processing applications the node set is typically a
regular pixel grid and $\cal E$ is induced by a e.g.\ 4-connected or
8-connected neighborhood structure. We will write $\sum_s$ and $\sum_{s \sim
  t}$ as shorthand notations for $\sum_{s \in \cal V}$ and $\sum_{(s,t) \in
  {\cal E}}$, respectively. Our convention is that $s$, $t$ denote nodes from
$\cal V$ and $i$, $j$ indicate states (or labels). We will also use the sets
$\mathrm{out}(s)$ for the successor of $s$ and $\mathrm{in}(s)$ for the
ancestor nodes of $s$.

The $d$-dimensional (unit or probability) simplex is defined as $\Delta^d
\defeq \{ x \ge 0: \sum_{i=0}^{d-1} x^i = 1 \}$. Elements $x \in \Delta^d$ can
be seen as discrete probability densities, and we denote the corresponding
cumulative distribution function $X$ with $X^i = \sum_{j=0}^{i-1} x^j$. We
extend $X$ to indices $i \in \mathbb{Z}$ with $X^i = 0$ for $i \le 0$ and $X^i
= 1$ for $i \ge L$. If $x \in \Delta^d$ is integral (e.g.\ $x^i = 1$ for some
$i$, and 0 otherwise), then $X$ can also be interpreted as superlevel function
with $X^j = [j > i]$. The main purpose of introducing $X^i$ is to have a
shorthand notation for $\sum_{j=0}^{i-1} x^j$, which will occur frequently in
this manuscript.

We use the notations $\imath_C(x)$ and $\imath\{ x \in C \}$ to write a
constraint $x \in C$ as an extended valued function, i.e.\ $\imath_C(x)$ is 0
iff $x \in C$ and $\infty$ otherwise. For a convex function $f$ we denote its
convex conjugate by $f^*$ and the l.s.c.\ extension of its perspective as
$\persp{f}$, which can be defined via the biconjugate:
\begin{align}
  \persp{f}(z, w) \defeq \max_{\mu, {\mbox{\scriptsize $\nu$}}: \mu + f^*({\mbox{\scriptsize $\nu$}}) \ge 0} z\mu + w^T \nu.
\end{align}
Throughout the manuscript we assume that the recession function of $f$ is
$\imath_{\{0\}}$ (i.e.\ $\lim_{z \to 0+}\persp{f}(z, w) =
\imath_{\{0\}}(w)$). This can be achieved by adding redundant bounds
constraints to $f$, since all unknowns in our convex problems are usually
restricted to $[0,1]$. In section~\ref{sec:metr} we will make use of the
following fact (see e.g.~\cite{zach2012dcmrf}):
\begin{lemma}
  \label{lem:persp}
  Let $\{f^i\}_{i = 1, \dotsc, n}$ be a family of convex functions, then
  \begin{align}
    \min_{\xi \in \mathbb{R}^d} \min_i f^i(\xi) = \min_{z \in \Delta^n} \min_{ w^i \in \mathbb{R}^d} \sum_{i=1}^n \persp{f}^i(z^i, w^i).
    \label{eq:convex_function_minimum}
  \end{align}
\end{lemma}

\subsection{Approximate Inference}

For a given graph ${\cal G} = ({\cal V}, {\cal E})$ and label (state) space
${\cal L} = \{0, \dotsc, L-1\}$ the task of inference in a label assignment
problem is to determine a minimizer of
\begin{align}
  \label{eq:elabeling}
  E_{\text{labeling}}(\Lambda) = \sum_{s} \theta_{s}^{\Lambda(s)} + \sum_{s \sim t} \theta_{st}^{\Lambda(s), \Lambda(t)},
\end{align}
where $\Lambda: {\cal V} \to {\cal L}$ (a mapping from nodes to states), and
$\theta_s^i$ and $\theta_{st}^{ij}$ are the unary and pairwise potentials,
respectively. The local nature of the modeled interactions means that the
above label assignment problem is an instance of a Markov Random Field
(MRF). Note that with respect to inference (i.e.\ finding a MAP solution)
conditional random fields (CRFs) are completely equivalent to MRFs, and our
use of the term ``MRF'' includes both CRFs and ``proper'' MRFs. In this work we focus on
problems with at most pairwise interactions between labels. In general such
labeling problems are difficult to solve exactly due to the NP-hardness of
many instances. A tractable approximation to $E_{\text{labeling}}$ is obtained
by ``lifting'' the problem to a higher dimensional setting: for each node $s
\in \cal V$ and each edge $(s,t) \in \cal E$ vectors $x_s \in \Delta^L$
and $x_{st} \in \Delta^{L^2}$ are introduced, which denote ``soft
one-hot'' encodings of the labels assigned to a node (or to an edge,
respectively) (see
e.g.~\cite{werner2007maxsum_review,wainwright2008graphical,sontag2011introduction}).
As result one obtains the following linear programming relaxation enabling
tractable approximate inference:
\begin{align}
E_{\text{LP-MRF}}(\vec x) &= \sum_{s, i} \theta_{s}^i x_s^i
+ \sum_{s \sim t} \sum_{i,j} \theta_{st}^{ij} x_{st}^{ij} \label{eq:lpmrf} \\
\text{subject to } & x_s^i = \sum_{j} x_{st}^{ij} \qquad x_t^j = \sum_i x_{st}^{ij} \nonumber \\
&x_s \in \Delta^L, \qquad x_{st}^{ij} \geq 0, \qquad \forall s,t,i,j
\nonumber
\end{align}
The first set of constraints are usually called \emph{marginalization
  constraints}. These constraints ensure that the labels assigned to edges are
consistent with the ones assigned at nodes. $x_s \in \Delta^L$ implies the
\emph{normalization constraint} $\sum_i x_s^i = 1$.

In many computer vision problems (e.g image denoising, optical flow) the
pairwise terms often do not depend on the actual labels $i$ and $j$ but only
on their difference (i.e.\ the ``height'' of jumps between labels). In this
case the pairwise potentials can be written as $\theta_{st}^{ij} =
\vartheta_{st}^{j-i}$, or even as $\theta_{st}^{ij} = \theta_{st}^{ji} =
\vartheta_{st}^{|i-j|}$ in the case of symmetric ones.

The number of unknowns in Eq.~\ref{eq:lpmrf} is in $O(L^2|{\cal E}|)$ which
can make inference with many labels costly. For certain pairwise potentials
the number of unknowns can be reduced to $O(L|{\cal E}|)$, e.g.\ $L^1$
potentials ($\theta_{st}^{ij} = w_{st} |i-j|$,
\cite{ishikawa1998segmentation,boykov98markov}), and truncated $L^1$
potentials ($\theta_{st}^{ij} = w_{st} \min\{ \tau_{st}, |i-j| \}$,
e.g.~\cite{zach2013connecting}). In this work we show that the number of
primal unknowns can be reduced from $O(L^2|{\cal E}|)$ to $O(KL|{\cal E}|)$
for piecewise linear potentials consisting of $K$ segments.

\section{The Main Result}
\label{sec:main_result}

In this section we state our main result, which generically shows that
piecewise linear pairwise potentials allow for a compact reformulation of
$E_{\text{LP-MRF}}$ (Eq.~\ref{eq:lpmrf}). In this section we only sketch the
proof, since it is based on more general constructions described in detail in
Sections~\ref{sec:conv} and~\ref{sec:metr}.
\begin{theorem}
  \label{thm:main}
  Let $\vartheta_{st}^{h} = \theta_{st}^{i,i+h}$ be a pairwise potential, that
  is a piecewise linear function with respect to $h$ consisting of $K$
  segments, having breakpoints only at integral values of $h$. Then there
  exists a reformulation of Eq.\ref{eq:lpmrf} that requires $2KL$ primal
  unknowns and $2L(K+1)+K$ linear constraints per edge in the graph.
\end{theorem}
\begin{proof}(Sketch:)
  In the following we consider a particular edge $(s,t)$ and drop the
  subscript $st$. Under the above assumptions on $\vartheta^h$ can be written
  as
  \begin{align}
    \vartheta^h = \min_{k \in \{0, \dots, K-1\} } \left\{ \alpha^k h + \beta^k + \imath_{[\underline{h}^k, \overline{h}^k]}(h) \right\},
    \label{eq:min_theta}
  \end{align}
  i.e.\ as minimum of linear functions with bounded (and convex) domains (see
  Fig.~\ref{fig:min_linear_bounded}). Using the results derived in
  Section~\ref{sec:conv}, potentials of the form
  \begin{align}
    \vartheta_k^h \defeq \alpha^k h + \beta^k + \imath_{[\underline{h}^k, \overline{h}^k]}(h)
    \label{eq:theta_k}
  \end{align}
  allow for a compact reformulation of Eq.~\ref{eq:lpmrf} using only $2L$
  primal unknowns and at most $2L+2$ constraints (see
  Eq.~\ref{eq:fst_linear2}). In Section~\ref{sec:metr} it is shown that the
  minimum of such $K$ potentials (represented by their compact reformulations)
  leads to a combined reformulation with $2LK$ unknowns and $2L(K+1)+K$
  constraints (see Eq.~\ref{eq:pmin_linear}).
\end{proof}

\begin{figure}[htb]
  \centering
  \subfigure[]{\includegraphics[height=10em]{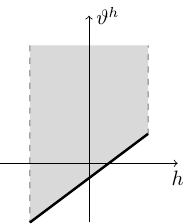}}
  \subfigure[]{\includegraphics[height=10em]{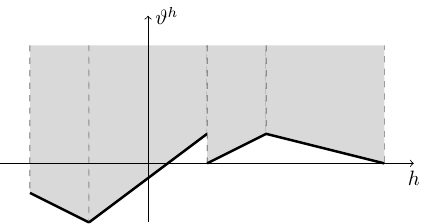}}
  \caption{(a) A bounded linear potential. (b) A piecewise linear (but not
    necessary continuous) potential represented as minimum of bounded linear
    ones.}
  \label{fig:min_linear_bounded}
\end{figure}

\begin{remark}
  Our way of counting constraints corresponds to the number of dual variables
  one needs to introduce in order to obtain a \emph{convenient} saddle-point
  formulation suitable for straightforward optimization e.g.\ via the
  (preconditioned) primal-dual method~\cite{pock2011diagonal}. Therefore we do
  not need to introduce dual variables whenever closed-form proximal steps are
  available. In particular, we do not include simple bounds constraints (such
  as non-negativity constraints) in our counting of constraints, but we always
  introduce enough dual variables to avoid non-trivial proximal steps.
\end{remark}

\begin{remark}
  We want to emphasis that rewriting piecewise linear potentials as minimum of
  bounded linear functions (as in Eq.~\ref{eq:min_theta}) also allows for
  efficient updates in message passing algorithms (such as belief propagation
  and its variants with guaranteed convergence). Efficient methods addressing
  (optionally truncated) $L^1$ and quadratic regularization are already
  presented in~\cite{felzenszwalb2006efficient}, and we can easily generalize
  their result: assume the pairwise potentials can be written as minimum of
  $K$ simple convex potentials as in Eq.~\ref{eq:min_theta}, then the lower
  envelope computation
  \begin{align}
    i \mapsto \min_j \left\{ \theta_t^j + \vartheta^{j-i} \right\}
  \end{align}
  can be done in $O(KL)$ time. This can be seen as follows: we rewrite the
  minimum envelope as
  \begin{align}
    i &\mapsto \min_k \min_j \left\{ \theta_t^j + \vartheta_k^{j-i} \right\} \nonumber \\
    &= \min_k \min_{j: \underline{h}^k \le j-i \le \overline{h}^k} \left\{ \theta_t^j + \alpha_k(j-i) \right\}, \nonumber
  \end{align}
  hence the minimum envelope can be computed in $O(KL)$ time if the inner
  envelope can be done in $O(L)$ time. Observe that the lower envelope
  \begin{align}
    \min_{j: \underline{h}^k \le j-i \le \overline{h}^k} \left\{ \theta_t^j + \alpha_k(j-i) \right\} \nonumber
  \end{align}
  is an instance of the min-filter problem, which can be solved in $O(L)$ time
  (e.g.~\cite{yuan2011running}). Interestingly, the very easily implementable
  online algorithm for min-filtering proposed in~\cite{lemire2006streaming}
  clearly resembles the lower envelope algorithm for quadratic costs
  in~\cite{felzenszwalb2006efficient}.
\end{remark}

\section{Piecewise Linear and Convex Pairwise Potentials}
\label{sec:conv}

In this section we consider pairwise potentials, that can be written as
pointwise maximum of affine functions in terms of the label difference $h =
j-i$, i.e.
\begin{align}
  \label{eq:convpl}
  \theta_{st}^{ij} = \vartheta_{st}^{j-i} = \max_{k \in \{0, \ldots, K-1\}} \left\{ \bar\alpha^k (j-i) + \bar\beta^k \right\}
\end{align}
for parameters $\bar\alpha^k$, $\bar\beta^k \in \mathbb{R}$. We assume that
the breakpoints of $\vartheta_{st}^{h}$ as a function of $h = j-i$ are located
on integral arguments $h$. Since pairwise potentials are only specified for
integral label values, this can be always achieved (at the expense of at most
doubling the number of affine functions).

In order to simplify the notation we assume w.l.o.g.\ edge-independent values
$\bar\alpha^k$ and $\bar\beta^k$ (and therefore drop the subscript $st$), but
all results below hold for edge-specific coefficients $\bar\alpha_{st}^k$ and
$\bar\beta_{st}^k$ as well. By definition $\vartheta_{st}^{h}$ is a convex and
piecewise linear function with respect to $h$, see also
Fig.~\ref{fig:maxmin}(a).

\subsection{Minimum Cut Graph Construction}

Our construction below is different to Ishikawa's graph cut approach solving
MRFs with convex and symmetric priors~\cite{ishikawa2003exact}, but can be
seen as generalization of his earlier construction
in~\cite{ishikawa1998segmentation}. The main benefits of our proposed
construction can be summarized as follows: first, it is very intuitive to
understand; second, it naturally allows asymmetric convex potentials; and
finally, it immediately enables extensions to more isotropic regularizers that
can be relevant in image processing applications.

The labeling problem with convex pairwise priors is solved by computing the
minimum-cut in a weighted graph. The node set of the graph is $\{ S, T\} \cup
\{a_s^i\}_{s \in {\cal V}, i \in \{0, \dotsc, L\}}$, where $S$ and $T$ are the
source and sink, respectively. The edge set contains infinity links $(S,
a_s^0)$ and $(a_s^L, T)$ for all $s \in \cal V$. Node $a_s^i$ is connected to
node $a_s^{i+1}$ with a directed edge $e_s^i$. A label $i$ is assigned to $s$
if the minimal cut goes through edge $e_s^i$. In order to ensure that only one
label is assigned at a node $s$, there are directed edges with infinite weight
from nodes $a_s^{i+1}$ to $a_s^i$. Finally, a subset of directed pairwise
edges connecting $a_s^i$ with $a_t^j$ (and vice versa) will be included into
the graph as described in the following.

The convex potential in Eq.~\ref{eq:convpl} can be equivalently written as
(with $h = j-i$)
\begin{align}
  \label{eq:convpl2}
  \vartheta_{st}^{h} = \sum_{k \in \{0, \dotsc, K-1\}} \left[ \gamma^k (h + \delta^k) \right]_+ + \alpha h + \beta.
\end{align}
It will be convenient later to explicitly include the affine term, $\alpha h +
\beta$. Note that $\beta$ only affects the objective of the minimizer, not the
minimizer itself, and hence can be ignored in the graph construction. The term
$\alpha h = \alpha j - \alpha i$ does not depend jointly on $i$ and $j$, and
consequently can be (temporarily) absorbed into the unary potentials for the
graph construction (e.g.\ $\theta_s^i$ and $\theta_t^j$ are augmented with
$-\alpha i$ and $\alpha j$, respectively). Thus, we focus on the first
expression in Eq.~\ref{eq:convpl2} below.

With our above assumption of integral breakpoints we have $\delta^k \in
\mathbb{Z}$ without loss of generality. In the following we focus on a single
summand, $\left[ \gamma^k (h + \delta^k) \right]_+$. If $\delta^k = 0$, the
term $\left[ \gamma^k h \right]_+$ corresponds to a ``one-sided'' $L^1$ (or
total variation) regularizer and can be solved by adding lateral directed
edges into the graph (see Fig.~\ref{fig:graph_cut}(a)). If $\delta^k \ne 0$,
one can temporarily reinterpret label value $j$ as $j + \delta^k$ and again
obtain an asymmetric $L^1$-type smoothness prior \emph{but between label $i$
  and a ``shifted'' label $j+\delta^k$}. Consequently, for each term $\left[
  \gamma^k (h + \delta^k) \right]_+$ with $\delta^k \ne 0$ directed diagonal
edges are inserted into the graph (see Fig.~\ref{fig:graph_cut}(b)).

\begin{figure}
  \centering
  \subfigure[Asymm.\ $L^1$]{\includegraphics[height=5.cm]{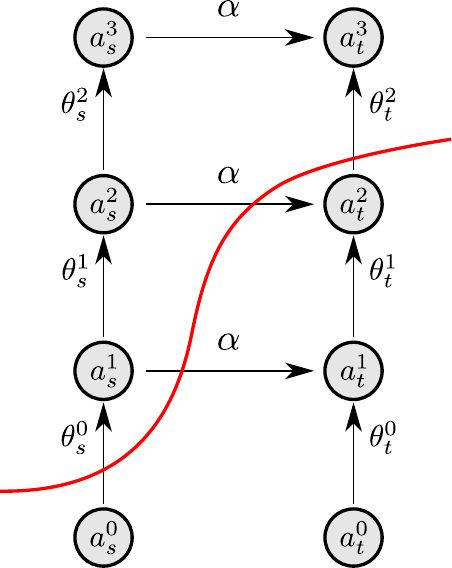}} \hspace{1em}
  \subfigure[``Shifted'' $L^1$]{\includegraphics[height=5.cm]{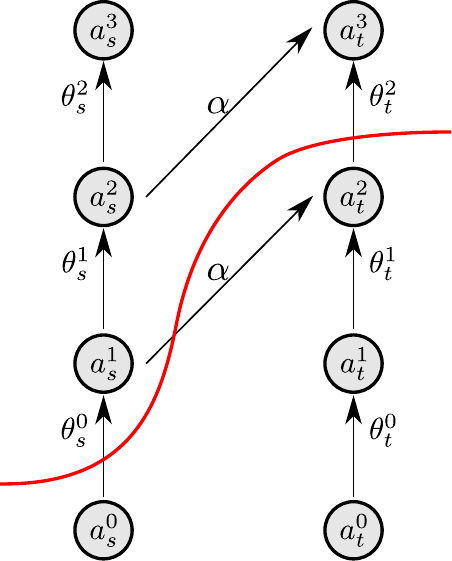}}
  \caption{Graph cut construction for asymmetric $L^1$ (a) and ``shifted''
    $L^1$ potentials (b). The red curve illustrates a potential cut. }
  \label{fig:graph_cut}
\end{figure}

\subsection{The Equivalent Linear Program}

The resulting graph can be immediately written as linear program\footnote{The
  expressions $[\xi]_+$ appearing in the objective can be removed leading to a
  proper LP after introducing a non-negative $\zeta$ with $\zeta \ge \xi$. },
\begin{align}
  E&_{\text{Cvx-Cut}}(\vec u) = \sum_{s,i} \theta_s^i \left(u_s^{i+1} - u_s^{i} \right) + \sum_{s \sim t} \beta \\
  &+ \sum_{s \sim t}\! \left( \alpha \sum_i (u_t^i - u_s^i)
  + \sum_{k,i} \left[ \gamma^k \left(u_s^i - u_t^{i+\delta^k} \right) \right]_+ \right) \nonumber \\
  \text{s.\ t. } & u_s^0 = 0, \qquad u_s^L = 1, \qquad u_s^{i+1} \geq u_s^i, \nonumber
\end{align}
where $\vec u: {\cal V} \times \{0, \dotsc, L\} \to [0, 1]$ encodes whether
node $a_s^i$ belongs to the source ($u_s^i=0$) or to the sink
($u_s^i=1$). Note that we explicitly state the contribution of the linear
part, $\bar\alpha h$ in Eq.~\ref{eq:convpl2}, to the unaries. In order to keep
the equations simple, we introduce the convention that ``out-of-bounds''
values $u_s^i$ yield 0 if $i < 0$ and 1 if $i \ge L$ for all $s \in \cal V$.

Observe that $u_s^i$ as function of $i$ is a superlevel representation and can
therefore be written as $u_s^i = X_s^i$ for some $x_s \in \Delta^{L}$. Thus,
we can rewrite $E_{\text{Cvx-Cut}}$ as
\begin{align}
  E_{\text{Cvx-LP}}(\vec x) & = \sum_{s,i} \theta_s^i x_s^i
  + \sum_{s \sim t, i} \sum_{k} \left[ \gamma^k \left(X_s^i - X_t^{i+\delta^k} \right) \right]_+ \nonumber \\
  &+ \sum_{s \sim t} \left( \alpha \sum_i (X_t^i - X_s^i) + \beta \right)
  \label{eq:cvx_lp} \\
  \text{s.t. } & x_s \in \Delta^L. \nonumber
\end{align}
In order to use this result as a building block for the construction in the
following Section~\ref{sec:metr}, we introduce for the pairwise terms
\begin{align}
  f_{st}^{\text{cvx}}&(y_s, y_t | x_s, x_t) \defeq \alpha_{st} \sum_i (Y_t^i - Y_s^i) + \beta_{st}
  \label{eq:fst_cvx} \\
  &+ \sum_{k,i} \left[ \gamma_{st}^k \left(Y_s^i - Y_t^{i+\delta_{st}^k} \right) \right]_+
  + \imath_{\Delta^L}(y_s) + \imath_{\Delta^L}(y_t). \nonumber
\end{align}
We explicitly added the (redundant) simplex constraints on $y_s$ and $y_t$ in
order to obtain a trivial recession function for $f_{st}^{\text{cvx}}$, which
will be important in Section~\ref{sec:metr}. Now, $E_{\text{Cvx-LP}}$ can be
rewritten as
\begin{align}
  E_{\text{Cvx-LP}}(\vec x) &=
  \sum_{s,i} \theta_s^i x_s^i + \sum_{s \sim t} \min_{\begin{smallmatrix} y_s:y_s = x_s \\ y_t:y_t = x_t \end{smallmatrix}} f_{st}^{\text{cvx}}(y_s, y_t | x_s, x_t)
  \label{eq:cvx_lp2}
\end{align}
subject to $x_s \in \Delta^L$. Observe that $f_{st}^{\text{cvx}}$ introduces
$2L$ unknowns, $y_s^i$ and $y_t^i$, per edge $(s,t)$, and enforces $2L+KL =
L(K+2)$ constraints (where we identify $[\cdot]_+$ with one inequality
constraint). Of course, the extra $2L$ unknowns, $y_s^i$ and $y_t^i$, and $2L$
of the constraints can be discarded immediately by applying e.g.\ the
constraint $y_s^i = x_s^i$, but $f_{st}^{\text{cvx}}$ as stated in
Eq.~\ref{eq:fst_cvx} will be important in Section~\ref{sec:metr}.

Overall, depending on the values of $K$ and $L$ optimization of
$E_{\text{Cvx-LP}}$ potentially requires far less memory than optimizing the
generic LP relaxation $E_{\text{LP-MRF}}$ (Eq.~\ref{eq:lpmrf}). A particular
and important instance of convex potentials are $L^1$-type ones, $h \mapsto
\alpha_{st} |h| + \beta_{st}$. For completeness we state one respective
specialization of $f_{st}^{\text{cvx}}$ to $f_{st}^{L^1}$:
\begin{align}
  f_{st}^{L^1}(y_s, y_t | x_s, x_t) &\defeq \alpha_{st} \sum_i \left| Y_t^i - Y_s^i \right| + \beta_{st} 
  + \imath_{\Delta^L}(y_s) + \imath_{\Delta^L}(y_t).
  \label{eq:fst_L1}
\end{align}

\subsection{Linear Priors with Bounded Domains}

In this section we discuss the particular prior relevant for the main result
in Section~\ref{sec:main_result}, where specific convex potentials of the
shape
\begin{align}
  \vartheta^h &= \alpha h + \beta + \imath_{[\underline{h}, \overline{h}]}(h) \nonumber \\
  &= \alpha h + \beta + M \left[ h - \overline{h} \right]_+ + M \left[ \underline{h} - h \right]_+
\end{align}
(with $M$ being a sufficiently large constant) are considered. In this
particular setting $f_{st}^{\text{cvx}}$ (Eq.~\ref{eq:fst_cvx}) reads as
\begin{align}
  f_{st}^{\text{linear}}&(y_s, y_t | x_s, x_t) \defeq \alpha \sum_i \left( Y_t^i - Y_s^i \right) + \beta
  \label{eq:fst_linear} \\
  &+ M \sum_i \left[ Y_s^i - Y_t^{i+\overline{h}} \right]_+ + M \sum_i \left[ Y_t^{i+\underline{h}} - Y_s^i \right]_+ \nonumber
\end{align}
subject to $y_s, y_t \in \Delta^L$. With $M \to \infty$ the penalizer terms
transform into constraints $Y_s^i \le Y_t^{i+u}$ and $Y_s^i \ge Y_t^{i+l}$
(which correspond to infinity links in the respective minimum-cut graph), and
$f_{st}^{\text{linear}}$ therefore equivalently reads as
\begin{align}
  f&_{st}^{\text{linear}}(y_s, y_t | x_s, x_t) \defeq \alpha \sum_i \left( Y_t^i - Y_s^i \right) + \beta
  \label{eq:fst_linear2} \\
  &+ \imath\left\{ Y_s^i \le Y_t^{i+\overline{h}},\, Y_s^i \ge Y_t^{i+\underline{h}} \right\}
  + \imath_{\Delta^L}(y_s) + \imath_{\Delta^L}(y_t), \nonumber
\end{align}
and by plugging $f_{st}^{\text{linear}}$ as $f_{st}^{\text{cvx}}$ into
Eq.~\ref{eq:cvx_lp2} we obtain a program with $2L$ unknowns per edge and at
most $2L+2$ constraints (or dual variables) as claimed in the proof of
Theorem~\ref{thm:main}.

\section{Minimum of Pairwise Potentials}
\label{sec:metr}

In this section we show that the (pointwise) minimum of compactly
representable pairwise potentials leads again to a compact representation of
the corresponding linear program. This applies e.g.\ to pairwise potentials
that are the minimum of (not necessarily symmetric) $L^1$-type pairwise priors
(see Fig.~\ref{fig:maxmin}(b)). It is well-known that $L^1$-type priors lead
to linear programs with $O(L)$ unknowns per edge
(e.g.~\cite{ishikawa1998segmentation}, or apply the result from the previous
section). A corollary from the construction presented in the following is,
that the minimum of $K$ $L^1$-type priors only requires $O(LK)$ primal
unknowns without loosening the convex relaxation, compared to $E_{\text{LP-MRF}}$. We will call the pointwise
minimum of pairwise ``elementary'' potentials a ``min-potential'' in the
following.


In order to show the equivalence of a compact linear program for
min-potentials with the standard relaxation $E_{\text{LP-MRF}}$
(Eq.~\ref{eq:lpmrf}) we proceed in two steps:
\begin{itemize}
\item Assume we are given a pairwise potential, that can be written as
  point-wise minimum of some elementary potentials (see e.g.\
  Fig~\ref{fig:point_vs_term}(a)). In Section~\ref{sec:relationfull} it is
  shown that such potentials can be reformulated as term-wise minimum of
  elementary potentials (as illustrated in Fig~\ref{fig:point_vs_term}(b)).
\item If the elementary potentials are chosen such that they have a compact
  reformulation (as the ones discussed in Section~\ref{sec:conv}), one
  can substitute the elementary potentials by their corresponding compact
  reformulation. In Section~\ref{sec:relationcompact} the equivalence of the
  resulting reformulation is shown, and some relevant examples are provided.
\end{itemize}
Overall, the equivalence of $E_{\text{LP-MRF}}$ (Eq.~\ref{eq:lpmrf}) with
compact reformulations is thus established.

\begin{figure}[htb]
  \centering
  \subfigure[Point-wise]{\includegraphics[width=0.3\textwidth]{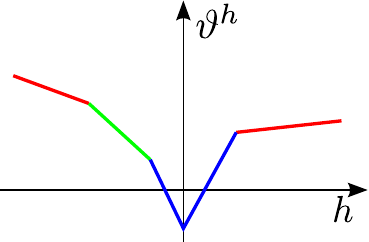}}
  \subfigure[Term-wise]{\includegraphics[width=0.3\textwidth]{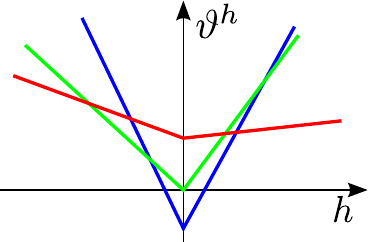}}
  \caption{Interpreting a pairwise potential either as point-wise minimum (a)
    or term-wise minimum (b). }
  \label{fig:point_vs_term}
\end{figure}
 
The following derivations make repeated use of Lemma~\ref{lem:persp}, which
allows us to rewrite a minimum of convex functions as convex minimization
problem in general. Application of Lemma~\ref{lem:persp} on the terms
representing elementary potentials in order to obtain min-potentials leads to
non-convex bilinear constraints as shown below. The following lemma states in
a general setting, that these bilinear constraints induced by application of
Lemma~\ref{lem:persp} can be ``linearized'' without affecting the minimum.
\begin{lemma}
  \label{lem:conveqiv}
  Let $\{f^k\}_{ k=0, \dotsc, K-1}$ be a family of convex l.s.c.\ functions
  (with trivial recession function). Further define the following minimization
  problems,
  \begin{align}
    F_0(w | \eta) & \defeq \min_k f^k(\eta, w) =\min_k \min_{y: y = \eta} f^k\left( y, w \right) \nonumber \\ \nonumber\\
    F_1(z, y, w) & \defeq \sum_k \persp{f^k}\left( z^k, (y^k, w^k) \right)
    +  \sum_k \imath \{ z^k \eta = y^k \} + \imath_{ \Delta }(z) \nonumber \\ \nonumber\\
    F_2(z, y, w) &\defeq \sum_k \persp{f^k}\left( z^k, (y^k, w^k) \right)
    + \imath\left\{ \sum\nolimits_k y^k = \eta  \right\} + \imath_{\Delta}(z). \nonumber
  \end{align}
  We have that
  \begin{equation}
     \min_w F_0(w | \eta) = \min_{z,y,w} F_1(z,y,w) = \min_{z,y,w} F_2(z,y,w).
  \end{equation}
\end{lemma}
\begin{proof}
  The equivalence of $F_0$ and $F_1$ follows immediately from
  Lemma~\ref{lem:persp}. The equality $\min F_1(z,y,w) = \min F_2(z,y,w)$ is
  shown as follows: observe that $\min_{z, y, w} F_2(z, y, w) \le \min_{z,
    y,w} F_1(z, y, w)$, since any $(z,y,w)$ feasible for $F_1$ is also
  feasible for $F_2$ (summing all constraints $z^k \eta = y^k$ over $k$ and
  using $\sum_k z^k = 1$ implies $\sum\nolimits_k y^k = \eta$).

   To prove $\min_{z, y, w} F_2(z, y, w) = \min_{z, y, w} F_1(z, y, w)$ we
   consider the Lagrangian duals of the two convex programs. We have (also
   applying Lemma~\ref{lem:persp})
  \begin{align}
    F_1^*(\lambda) &= \!\!\! \min_{z \in \Delta, y, w} \sum_k \left( \persp{f^k}(z^k, (y^k,w^k)) \! +\! (\lambda^k)^T \! (z^k \eta - y^k) \right) \nonumber \\
    &= \min_k \left\{ \min_{\zeta, \xi} f^k(\zeta,\xi) + (\lambda^k)^T \eta - (\lambda^k)^T \zeta \right\} \nonumber \\
    &= \min_k \left\{ \eta^T \lambda^k - \max_{\zeta,\xi} \left\{ \zeta^T \lambda^k - f^k(\zeta,\xi) \right\} \right\} \nonumber \\
    &= \min_k \left\{ \eta^T \lambda^k - (f^k)^*(\lambda^k, \mathbf{0}) \right\}. \nonumber
  \end{align}
  Analogously we obtain
  \begin{align}
    F_2^*(\nu) &= \min_{z \in \Delta, y,w} \sum_k \persp{f^k}(z^k, (y^k, w^k)) \!+ \!\nu^T \!\left( \eta - \sum\nolimits_k y^k \right) \nonumber \\
    &= \nu^T \eta + \min_k \min_{\zeta,\xi} \left\{ f^k(\zeta,\xi) - \nu^T \zeta \right\} \nonumber \\
    &= \nu^T \eta + \min_k \left\{ - \max_{\zeta,\xi} \left\{\nu^T \zeta - f^k(\zeta,\xi) \right\} \right\} \nonumber \\
    &= \nu^T \eta + \min_k \left\{- (f^k)^*(\nu, \mathbf{0}) \right\} \nonumber \\
    &= \min_k \left\{ \eta^T \nu - (f^k)^*(\nu, \mathbf{0}) \right\} \nonumber.
  \end{align}
  Both dual programs have essentially the same objective $\lambda \mapsto
  \min_k \left\{ \eta^T \lambda^k - (f^k)^*(\lambda^k, \mathbf{0}) \right\}$,
  but $F_2^*$ enforces additional constraints on its argument ($\lambda =
  (\nu, \dotsc, \nu)$ for some $\nu$), from which the already known fact
  $\max_\lambda F_1^*(\lambda) \ge \max_\nu F_2^*(\nu)$ follows. But any
  maximizer $\lambda^*$ of $F_1^*$ can be converted to a feasible solution of
  $F_2^*$ with the same objective by setting $\nu^* = (\lambda^*)^l$, where
  \begin{align}
    l \in \arg\min_k \left\{ \eta^T \lambda^k - (f^k)^*(\lambda^k, \mathbf{0}) \right\}. \nonumber
  \end{align}
  Therefore both the dual programs have the same optimal value, and
  $\min F_1 = \min F_2$ follows from strong duality.
\end{proof}

In the following we repeatedly apply Lemma~\ref{lem:conveqiv} to obtain
compact and convex reformulations for min-potentials. The node variables $x_s$
and $x_t$ involved in the pairwise potentials via the marginalization
constraints (or a respective variant) will attain the role of $\eta$ in the
lemma. Since these node variables are subject to optimization as well (on the
outer scope), Lemma~\ref{lem:conveqiv} is important to replace occuring
bilinear constraints by linear ones.

\subsection{Term-Wise Minimum of Potentials}
\label{sec:relationfull}

Let the pairwise potentials $\theta_{st}^{ij}$ be written as pointwise minimum
of elementary potentials $\theta_{st}^{ij}$, i.e.
\begin{align}
  \theta_{st}^{ij} = \min_{k \in \{0, \dotsc, K-1\} } \theta_{st}^{ijk}. \nonumber
\end{align}
In this section we show the equivalence of
\begin{align}
  E_{\text{term-wise}}(\vec x) &= \sum_{s,i} \theta_{s}^i x_s^i
  + \sum_{s \sim t} \min_k \left\{ \sum\nolimits_{ij}  \theta_{st}^{ijk} x_{st}^{ijk} \right\} \\
  \text{s.\ t. } & x_s^i = \sum_{jk} x_{st}^{ijk} \qquad x_t^j = \sum_{ik} x_{st}^{ijk} \nonumber \\
  & x_s \in \Delta^L \qquad x_{st}^{ijk} \geq 0 \nonumber
\end{align}
and
\begin{align}
  E_{\text{point-wise}}(\vec x) &= \sum_{s,i} \theta_{s}^i x_s^i
  + \sum_{s \sim t} \sum_{ij} \min_k \theta_{st}^{ijk} x_{st}^{ij} \\
  \text{s.\ t. } & x_s^i = \sum_{j} x_{st}^{ij} \qquad x_t^j = \sum_i x_{st}^{ij} \nonumber \\
  & x_s \in \Delta^L \qquad x_{st}^{ij} \geq 0 \nonumber
\end{align}
Note that the difference between the two programs is the position of
$\min_k$. In $E_{\text{point-wise}}$ a new pairwise potential
$\breve\theta_{st}^{ij}$ is formed as the point-wise minimum of given ones,
\begin{align}
  \breve\theta_{st}^{ij} \defeq \min_k \theta_{st}^{ijk},
\end{align}
whereas in $E_{\text{term-wise}}$ a particular active potential is selected
per edge (see also Fig.~\ref{fig:point_vs_term}). The equivalence of the two
energies is an immediate consequence of the following fact:

\begin{lemma}
  \label{lem:point_wise}
  For given $x_s$ and $x_t \in \Delta^L$ let ${\cal M}(x_s, x_t)$ denote the
  (local) marginalization constraints
  \begin{align}
    {\cal M}(x_s, x_t) \defeq \left\{ x_{st} \in \Delta^{L^2}:
    x_s^i = \sum\nolimits_{j} x_{st}^{ij}, x_t^j = \sum\nolimits_{i} x_{st}^{ij} \right\}. \nonumber
  \end{align}
  The following two pairwise costs
  \begin{align}
    p_{st}(x_{st} | x_s, x_t) &= \min_k \left\{ \sum\nolimits_{ij} \theta_{st}^{ijk} x_{st}^{ij} \right\} + \imath_{{\cal M}(x_s, x_t)}(x_{st}) \nonumber
  \end{align}
  and
  \begin{align}
    q_{st}(x_{st} | x_s, x_t) &= \sum_{i,j} \min_k \theta_{st}^{ijk} x_{st}^{ij} + \imath_{{\cal M}(x_s, x_t)}(x_{st}) \nonumber
  \end{align}
  are equivalent, i.e.
  \begin{align}
    \min_{x_{st}} p_{st}(x_{st} | x_s, x_t) = \min_{x_{st}} q_{st}(x_{st} | x_s, x_t).
  \end{align}
\end{lemma}
\begin{proof}
  First we remark that obviously $\min q_{st} \le \min p_{st}$, since a sum of
  pointwise minima is never larger than the minimum of sums (over the same
  terms). In order to show $\min q_{st} \ge \min p_{st}$ we rewrite
  \begin{align}
    p_{st}(x_{st} | x_s, x_t) = \min_k p_{st}^k(x_{st} | x_s, x_t) \nonumber
  \end{align}
  with
  \begin{align}
    p_{st}^k(x_{st} | x_s, x_t) \defeq \sum_{ij} \theta_{st}^{ijk} x_{st}^{ij} + \imath_{\Delta^{L^2}}(x_{st}). \nonumber
  \end{align}
  Observe that the recession function of $p_{st}^k$ is trivial due to the
  simplex constraint $x_{st} \in \Delta^{L^2}$. Thus, we can apply
  Lemma~\ref{lem:conveqiv} on $p_{st} = \min_k p_{st}^k$ to obtain an
  equivalent convex program,
  \begin{align}
    \tilde p_{st}&(z, y | x_s, x_t) = \sum_{ijk} \theta_{st}^{ijk} y^{ijk} + \imath_{\Delta^K}(z) + \imath_{\ge 0}(y) \nonumber \\
    & + \imath\left\{ \sum\nolimits_{ij} y^{ijk} = z^k,\,
    x_s^i = \sum\nolimits_{jk} y^{ijk},\, x_t^j = \sum\nolimits_{ik} y^{ijk} \right\}.
    \nonumber
  \end{align}
  We substitute $z^k = \sum_{ij} y^{ijk}$ in $\tilde p_{st}$ (leading to the
  constraint $1 = \sum_k z^k = \sum_{ijk} y^{ijk}$, or combined with $y \ge
  0$, to $y \in \Delta^{KL^2}$),
  \begin{align}
    \tilde p_{st}(y | x_s, x_t) &= \sum_{ijk} \theta_{st}^{ijk} y^{ijk} + \imath_{\Delta^{KL^2}}(y) \nonumber \\
    & +\imath\left\{ x_s^i = \sum\nolimits_{jk} y^{ijk},\, x_t^j = \sum\nolimits_{ik} y^{ijk} \right\}.
    \nonumber
  \end{align}
  Let $x_{st}^*$ be a minimizer of $q_{st}$, and let $k^{ij} \defeq \arg\min_k
  \theta_{st}^{ijk}$. We set $(y^*)^{ijk} = (x_{st}^*)^{ij}$ iff
  $k=k^{ij}$ and 0 otherwise. Then $y^*$ is feasible for $\tilde p_{st}$
  and has the same objective value as $q_{st}(x_{st}^*)$. Consequently,
  $\min_{y} \tilde p_{st} \le \min_{x_{st}} q_{st}$, and $\min p_{st} \le
  \min q_{st}$ since $p_{st}$ and $\tilde p_{st}$ are equivalent. Overall, we
  have $\min p_{st} = \min q_{st}$ as claimed.
\end{proof}
The equivalence of $E_{\text{point-wise}}$ and $E_{\text{term-wise}}$ means
that given pairwise potentials $\theta_{st}^{ij}$ than can be written as
pointwise minimum of some ``convenient'' elementary potentials, we can focus
our attention on compact representations of these elementary potentials.

\subsection{Term-Wise Minimum of Compact Potentials}
\label{sec:relationcompact}

In this section we assume that elementary pairwise potentials
$\theta_{st}^{ijk}$ have a more compact equivalent, e.g.
\begin{align}
  \min_{x_{st} \in {\cal M}(x_s, x_t)} \sum_{ij} \theta_{st}^{ijk} x_{st}^{ij}
  = \min_{\begin{smallmatrix} y,w \\ y_s = x_s \\ y_t = x_t \end{smallmatrix}} f_{st}^k(y, w | x_s, x_t) \nonumber
\end{align}
for some convex function $f_{st}^k$ (having a trivial recession function). The
respective non-convex labeling energy reads as
\begin{align}
  \label{eq:pmin}
  E_{\text{min-prior}}(\vec x) &= \sum_{s,i} \theta_s^i x_s^i + \sum_s \imath \{ x_s  \in \Delta^L \} \\
  & + \sum_{s \sim t}\min_{k \in \{ 0, \ldots, K-1 \}}\min_{\begin{smallmatrix} y,w \\ y_s = x_s \\ y_t = x_t \end{smallmatrix}} f^k_{st} (y, w | x_s, x_t),
  \nonumber
\end{align}
and we use Lemma~\ref{lem:conveqiv} below to obtain an equivalent convex
program. For concreteness, and due to the important of piecewise convex
pairwise potentials, we instantiate $f_{st}^k$ with $f_{st}^{\text{cvx}}$
(recall Eq.~\ref{eq:fst_cvx}) in our derivation,
\begin{align}
  f_{st}^k&(y_s, y_t | x_s, x_t) \defeq \alpha_{st}^k \sum_i (Y_t^i - Y_s^i) + \beta_{st}^k
  \label{eq:f_st_k} \\
  &+ \sum_{l,i} \left[ \gamma_{st}^{kl} \left(Y_s^i - Y_t^{i+\delta_{st}^{kl}} \right) \right]_+ + \imath_\Delta(y_s) + \imath_\Delta(y_t)
  \nonumber
\end{align}
In order to apply Lemma~\ref{lem:conveqiv} we need the perspective of
$f_{st}^k$, which can be immediately stated as
\begin{align}
  \persp{(f_{st}^k)}&(z^k, y_s^k, y_t^k | x_s, x_t) \defeq \alpha_{st}^k \sum_i (Y_t^{ki} - Y_s^{ki}) + \beta_{st}^k z^k \nonumber \\
  &+ \sum_{l,i} \left[ \gamma_{st}^{kl} \left(Y_s^{ki} - Y_t^{k,i+\delta_{st}^{kl}} \right) \right]_+ \nonumber \\
  &+ \imath\left\{ \sum\nolimits_i y_s^{ki} = \sum\nolimits_i y_t^{ki} = z^k \right\} + \imath_{\ge 0}(y).
\end{align}
Application of Lemma~\ref{lem:conveqiv} (with the constraints $y_s = x_s$ and
$y_t = x_t$ taking the role of the constraint ``$y=\eta$'') establishes the
equivalence of
\begin{align}
  \min_k \min_{y_s, y_t} f_{st}^k(y_s, y_t | x_s, x_t) \nonumber
\end{align}
and
\begin{align}
  \min_{z \in \Delta^K} \min_{y \ge 0} \sum_k \bigg( & \alpha_{st}^k \sum_i (Y_t^{ki} - Y_s^{ki}) + \beta_{st}^k z^k
  + \sum_{l,i} \left[ \gamma_{st}^{kl} \left(Y_s^{ki} - Y_t^{k,i+\delta_{st}^{kl}} \right) \right]_+ \bigg)
\end{align}
subject to
\begin{align}
  x_s^i &= \sum_k y_s^{ki} &  x_t^i &= \sum_k y_t^{ki} \nonumber\\
  z^k &= \sum_i y_s^{ki} &  z^k &= \sum_i y_t^{ki}. \nonumber
\end{align}
The occuring variables have intuitive meanings: $z$ is a (soft) one-hot
encoding of which branch $k$ is active in $\min_k f_{st}^k$, i.e.\ $z$
represents the set $\arg\min_k f_{st}^k$. It is easy to see that for all
values of $x_s$ and $x_t$ a minimizer $z$ will attain only binary values (it
can also be fractional if $\arg\min_k f_{st}^k$ contains more than one
element).  $y_s^{ki}$ and $y_t^{ki}$ represent $x_s^i$ and $x_t^i$,
respectively (as ``local copies''), in the $k$-th branch.

If the above expressions are plugged into Eq.~\ref{eq:pmin}, and the
edge-specific unknowns $z^k$ etc.\ are augmented with the respective edge
subscript $st$, one obtains the energy given in Eq.~\ref{eq:pmin_full}. If we
specialize $f_{st}^k$ to be the linear but bounded potentials
$f_{st}^{\text{linear}}$ (Eq.~\ref{eq:fst_linear}) and express $z$ in terms of
$y$ (e.g.\ $z_{st}^k = \frac{1}{2} \left( \sum_i y_{st \to s}^{ki} + \sum_i
y_{st \to t}^{ki} \right)$), we arrive at the convex program given in
Eq.~\ref{eq:pmin_linear} relevant for the main result in
Theorem~\ref{thm:main}. One can directly read off the number of unknowns per
edge (which is $2KL$) and constraints (dual variables, at most $2L + K + 2KL =
2L(K+1)+K$) from the resulting program.

For practical implementations it can be beneficial to implement
specializations of $E_{\text{min-prior}}$ rather than
$E_{\text{min-linear}}$. For instance, if the smoothness prior is the minimum
of $K$ $L^1$-type potentials, a generic implementation based on
$E_{\text{min-cvx}}$ (which models the potential via $2K$ linear segments)
requires about twice the number of unknowns and constraints than a specific
formulation $E_{\text{min-}L^1}$ (depicted in Eq.~\ref{eq:pmin_L1}) derived
from $E_{\text{min-prior}}$ and $f_{st}^{L1}$
(Eq.~\ref{eq:fst_L1}).\footnote{If we assume the solver can directly cope with
  $|\cdot|$, which is the case e.g.\ for proximal methods-based
  implementations. }

\begin{figure*}[htb]
  \centering
  \begin{align}
    E_{\text{min-cvx}}(\vec x, \vec y, \vec z) &= \sum_{s,i} \theta_s^i x_s^i + \sum_{s \sim t} \sum_k
    \left( \alpha_{st}^k \sum_i \left( Y_{st \to t}^{ki} - Y_{st \to s}^{ki} \right) + \beta_{st}^k z_{st}^k
    + \sum_{l,i} \left[ \gamma_{st}^{kl} \left(Y_{st \to s}^{ki} - Y_{st \to t}^{k,i+\delta_{st}^{kl}} \right) \right]_+ \right)
    \label{eq:pmin_full}
  \end{align}
  \begin{align}
    \text{s.t.} & & x_s^i &= \sum_k y_{st \to s}^{ki} &  x_t^i &= \sum_k y_{st \to t}^{ki} \nonumber \\
    & & z_{st}^k &= \sum_i y_{st \to s}^{ki} & z_{st}^k &= \sum_i y_{st \to t}^{ki}
    & x_s &\in \Delta^L, \, z_{st} \in \Delta^K,\, \vec{y} \ge 0.
    \nonumber
  \end{align}
  \caption{The convex relaxation for MRFs with pairwise potentials, that are
    the point-wise minima of piecewise linear convex ones.}
\end{figure*}

\begin{figure*}[htb]
  \centering
  \begin{align}
    E_{\text{min-linear}}(\vec x, \vec y) &= \sum_{s,i} \theta_s^i x_s^i + \sum_{s \sim t} \sum_k
    \left( \alpha_{st}^k \sum_i \left( Y_{st \to t}^{ki} - Y_{st \to s}^{ki} \right)
    + \frac{\beta_{st}^k}{2} \left( \sum_i y_{st \to s}^{ki} + \sum_i y_{st \to t}^{ki} \right) \right)
    \label{eq:pmin_linear}
  \end{align}
  \begin{align}
    \text{s.t.} & & x_s^i &= \sum_k y_{st \to s}^{ki} &  x_t^i &= \sum_k y_{st \to t}^{ki}
    & \sum_i y_{st \to s}^{ki} &= \sum_i y_{st \to t}^{ki} \nonumber \\
    & & Y_{st\to s}^{ki} &\le Y_{st \to t}^{k,i+\overline{h}} & Y_{st\to s}^{ki} &\ge Y_{st \to t}^{k,i+\underline{h}}
    & x_s &\in \Delta^L, \, z_{st} \in \Delta^K,\, \vec{y} \ge 0.
    \nonumber
  \end{align}
  \caption{The specialization of $E_{\text{min-cvx}}$ to bounded linear
    potentials that is used in Theorem~\ref{thm:main}. }
\end{figure*}

\begin{figure*}[htb]
  \centering
  \begin{align}
    E_{\text{min-}L^1}(\vec x, \vec y) &= \sum_{s,i} \theta_s^i x_s^i + \sum_{s \sim t} \sum_k
    \left( \alpha_{st}^k \sum_i \left|Y_{st \to s}^{ki} - Y_{st \to t}^{ki}\right|
    + \frac{\beta_{st}^k}{2} \left( \sum_i y_{st \to s}^{ki} + \sum_i y_{st \to t}^{ki} \right) \right)
    \label{eq:pmin_L1}
  \end{align}
  \begin{align}
    \text{s.t.} & & x_s^i &= \sum_k y_{st \to s}^{ki} &  x_t^i &= \sum_k y_{st \to t}^{ki} \nonumber \\
    & & \sum_i y_{st \to s}^{ki} &= \sum_i y_{st \to t}^{ki}
    & x_s &\in \Delta^L, \, \vec{y} \ge 0.
    \nonumber
  \end{align}
  \caption{The specialization of $E_{\text{min-cvx}}$ to the minimum of
    $L^1$-type potentials. }
\end{figure*}

\section{Reducing the Grid Bias}
\label{sec:isotropic}

Until now we restricted the exposition to labeling tasks with underlying
discrete (i.e.\ graph structured) domains. In some cases continuously inspired
label assignment formulations (e.g.~\cite{chambolle2012convex}) are preferable
in image processing and computer vision problems due to the reduced
metrication artifacts. As pointed out
in~\cite{zach2012multilabel,zach2013connecting} the finite difference
discretization of continuous formulations is closely related to standard LP
relaxations for inference such as $E_{\text{LP-MRF}}$ (Eq.~\ref{eq:lpmrf}). In
a nutshell, continuously inspired labeling formulation replace linear
smoothness terms (i.e.\ $\sum_{ij} \theta_{st}^{ij} x_{st}^{ij}$ in
$E_{\text{LP-MRF}}$) with nonlinear, Euclidean-norm based terms in order to
achieve better counting of non grid-aligned discontinuities.

\begin{figure}[htb]
  \centering
  \includegraphics[width=0.2\linewidth]{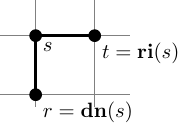}
  \caption{The forward difference stencil used in our discretization of the
    image plane. }
  \label{fig:stencil}
\end{figure}

We follow existing literature and use a forward difference stencil in the
following. We expect slightly improved visual results if the discretization of
the image plane is based on e.g.\ a staggered grid or uses a discrete calculus
formulation~\cite{grady2010discrete}.  In our finite difference setting the
image domain is represented by a regular grid with horizontal and vertical
edges between pixels. We index horizontal edges with subscripts $(s,h)$ (where
$s$ is the left grid point of the edge) and vertical ones with $(s,v)$ (see
also Fig.~\ref{fig:stencil}). Thus, e.g.\ variables $y_{s,h \to s}^{ki}$ and
$y_{s,v \to s}^{ki}$ represent $y_{st \to s}^{ki}$ depending whether $(s,t)$
is a horizontal or vertical edge. For notational simplicity we assume
homogeneous and symmetric pairwise potentials in the following (i.e.\ we drop
edge subscripts in the coefficients, and assume $\alpha^k = 0$). The simplest
isotropic extension shown in Eq.~\ref{eq:pmin_full_iso} replaces the separate,
$L^1$-type counting of horizontal and vertical discontinuities in
Eq.~\ref{eq:pmin_full}, e.g.\ terms such as
\begin{align}
  &\sum_{l,i} \left[ \gamma^{kl} \left(Y_{s,h \to s}^{ki} - Y_{s,h \to t}^{k,i+\delta^{kl}} \right) \right]_+
  + \sum_{l,i}\left[ \gamma^{kl} \left(Y_{s,v \to s}^{ki} - Y_{s,v \to t}^{k,i+\delta^{kl}} \right) \right]_+ \nonumber \\
  = & \sum_{l,i} \left\| \begin{matrix}
    \left[ \gamma^{kl} \left(Y_{s,h \to s}^{ki} - Y_{s,h \to t}^{k,i+\delta^{kl}} \right) \right]_+ \\
    \left[ \gamma^{kl} \left(Y_{s,v \to s}^{ki} - Y_{s,v \to t}^{k,i+\delta^{kl}} \right) \right]_+
  \end{matrix} \right\|_1
  \nonumber
\end{align}
with a joint Euclidean-norm penalizer,
\begin{align}
  \sum_{l,i} \left\| \begin{matrix}
    \left[ \gamma^{kl} \left(Y_{s,h \to s}^{ki} - Y_{s,h \to t}^{k,i+\delta^{kl}} \right) \right]_+ \\
    \left[ \gamma^{kl} \left(Y_{s,v \to s}^{ki} - Y_{s,v \to t}^{k,i+\delta^{kl}} \right) \right]_+
  \end{matrix} \right\|_2.
  \nonumber
\end{align}
Consequently, edges cut jointly in horizontal and vertical direction imply a
$\sqrt{2}$ increase in the smoothness term, which corresponds to the standard
$\sqrt{2}$-length penalization of diagonal discontinuities.

Similar reasoning holds for the constant cost $\beta^k$ in $f_{st}^k$
(Eq.~\ref{eq:f_st_k}, which translates to the term $\beta^k z_{st}^k$ in
Eq.~\ref{eq:pmin_full}). We also replace the $L^1$-type contribution
\begin{align}
  \beta^k \left( z_{s,h}^k + z_{s,v}^k \right) = \beta^k \left\| \begin{matrix} z_{s,h}^k \\ z_{s,v}^k \end{matrix} \right\|_1 \nonumber
\end{align}
(since $\vec z \ge 0$) with an Euclidean cost,
\begin{align}
  \beta^k \left\| \begin{matrix} z_{s,h}^k \\ z_{s,v}^k \end{matrix} \right\|_2,
  \nonumber
\end{align}
and the complete objective is given in $E_{\text{min-cvx}}^{\text{isotr.}}$
(Eq.~\ref{eq:pmin_full_iso}). In Eq.~\ref{eq:pmin_L1_iso} we further depict
the convex program in analogy to the anisotropic energy $E_{\text{min-}L^1}$
(Eq.~\ref{eq:pmin_L1}).

Our choice for the isotropic extension reduces to approaches presented in the
literature such as isotropic $L^1$ potentials~\cite{pock2010global}, Potts
smoothness prior~\cite{zach2008labeling}, and truncated
priors~\cite{zach2013connecting}. The construction is described for 2D image
domains but can obviously be extended to higher dimensional image domains.

\begin{figure*}[htb]
  \centering
  \begin{align}
    E_{\text{min-cvx}}^{\text{isotr.}}(\vec x, \vec y, \vec z) &= \sum_{s,i} \theta_s^i x_s^i + \sum_{s,k}
    \left( \beta^k \left\| \begin{matrix} z_{s,h}^k \\ z_{s,v}^k \end{matrix} \right\|_2
    + \sum_{l,i} 
    \left\| \begin{matrix}
      \left[ \gamma^{kl} \left( Y_{s,h \to s}^{ki} - Y_{s,h \to t}^{k,i+\delta^{kl}} \right) \right]_+ \\
      \left[ \gamma^{kl} \left( Y_{s,v \to s}^{ki} - Y_{s,v \to t}^{k,i+\delta^{kl}} \right) \right]_+
    \end{matrix} \right\|_2 \right)
    \label{eq:pmin_full_iso}
  \end{align}
  \begin{align}
    \text{s.t.} & & x_s^i &= \sum_k y_{s,h \to s}^{ki} &  x_{\mathsf{ri}(s)}^i &= \sum_k y_{s,h \to t}^{ki} 
    & x_s^i &= \sum_k y_{s,v \to s}^{ki} &  x_{\mathsf{dn}(s)}^i &= \sum_k y_{s,v \to t}^{ki} \nonumber \\
    & & z_{s,h}^k &= \sum_i y_{s,h \to s}^{ki} & z_{s,h}^k &= \sum_i y_{s,h \to t}^{ki}
    & z_{s,v}^k &= \sum_i y_{s,v \to s}^{ki} & z_{s,v}^k &= \sum_i y_{s,v \to t}^{ki} \nonumber \\
    & & & & & & & & x_s &\in \Delta^L, \, z_{s,h}, z_{s,v} \in \Delta^K,\, \vec{y} \ge 0.
    \nonumber
  \end{align}
  \caption{A convex relaxation for MRFs with symmetric priors reducing the grid bias.}
\end{figure*}

\begin{figure*}[htb]
  \centering
  \begin{align}
    E_{\text{min-}L^1}^{\text{isotr.}}(\vec x, \vec y, \vec z) &= \sum_{s,i} \theta_s^i x_s^i + \sum_{s,k}
    \left( \beta^k \left\| \begin{matrix} z_{s,h}^k \\ z_{s,v}^k \end{matrix} \right\|_2
    + \alpha^k \sum_i
    \begin{Vmatrix}
      Y_{s,h \to s}^{ki} - Y_{s,h \to t}^{ki} \\ Y_{s,v \to s}^{ki} - Y_{s,v \to t}^{ki} 
    \end{Vmatrix}_2 \right)
    \label{eq:pmin_L1_iso}
  \end{align}
  subject to the same constraints as in Eq.~\ref{eq:pmin_full_iso}.
  \caption{The specialization of $E_{\text{min-cvx}}^{\text{isotr.}}$ to
    $L^1$-type potentials.}
\end{figure*}

\begin{remark}
  In order to reduce the metrication artifacts for min-potentials one has two
  options: the first option is to convert an anisotropic formulation,
  e.g.\ Eq.~\ref{eq:pmin_full}, to behave ``less anisotropic.'' This is the
  approach as presented above.  The other option is to use isotropic
  formulations as elementary potentials and subsequently apply the
  construction described in Section~\ref{sec:metr} to obtain the minimum
  potential. If we focus on the minimum of $L^1$-type (total variation)
  potentials, and if we employ the standard forward-difference discretization,
  the underlying elementary potential is given by
  \begin{align}
      f^{TV}(y_s, y_t, y_r) &\defeq \alpha \sum_i \begin{Vmatrix} Y_s^i - Y_t^i \\ Y_s^i - Y_r^i \end{Vmatrix}_2 + \beta
      + \imath_{\Delta^L}(y_s) + \imath_{\Delta^L}(y_t) + \imath_{\Delta^L}(y_r),
      \label{eq:fst_TV}
  \end{align}
  where nodes $s$, $t$, and $r$ form a neighborhood as shown in
  Fig.~\ref{fig:stencil}. Application of Lemma~\ref{lem:conveqiv} to rewrite
  terms
  \begin{align}
    \min_k \left\{ \alpha^k \sum_i \begin{Vmatrix} Y_s^i - Y_t^i \\ Y_s^i - Y_r^i \end{Vmatrix}_2 + \beta^k \right\}
  \end{align}
  leads to the convex program Eq.~\ref{eq:pmin_L1_iso_b}. It can be observed
  from the resulting constraints, that in this formulation e.g.\ the same
  elementary potential needs to be selected in horizontal \emph{and} vertical
  direction. Further, the ``constants'' $\beta^k$ are counted differently in
  $E_{\text{min-$L^1$}}^{\text{isotr.}}$ and
  $E_{\text{min-$L^1$-b}}^{\text{isotr.}}$. We currently prefer
  $E_{\text{min-$L^1$}}^{\text{isotr.}}$, since it reduces to the Potts and
  truncated $L^1$ models presented in the literature, but a deeper analysis is
  subject of future research.
\end{remark}

\begin{figure*}[htb]
  \centering
  \begin{align}
    E_{\text{min-$L^1$-b}}^{\text{isotr.}}(\vec x, \vec y, \vec z) &= \sum_{s,i} \theta_s^i x_s^i + \sum_{s,k}
    \left( \beta^k z_s^k + \alpha^k \sum_i
    \begin{Vmatrix}
      Y_{s \to s}^{ki} - Y_{s \to t}^{ki} \\ Y_{s \to s}^{ki} - Y_{s \to r}^{ki} 
    \end{Vmatrix}_2 \right)
    \label{eq:pmin_L1_iso_b}
  \end{align}
  \begin{align}
    \text{s.t.} & & x_s^i &= \sum_k y_{s \to s}^{ki} &  x_{\mathsf{ri}(s)}^i &= \sum_k y_{s \to t}^{ki} 
    & x_{\mathsf{dn}(s)}^i &= \sum_k y_{s \to r}^{ki} \nonumber \\
    & & z_{s}^k &= \sum_i y_{s \to s}^{ki} & z_{s}^k &= \sum_i y_{s \to t}^{ki} & z_{s}^k &= \sum_i y_{s \to r}^{ki}
    & x_s &\in \Delta^L, \, z_{s}\in \Delta^K,\, \vec{y} \ge 0.
    \nonumber
  \end{align}
  \caption{An alternative formulation to reduce the grid bias of
    $E_{\text{min-}L^1}$.}
\end{figure*}

\section{Numerical Results}
\label{sec:experiment}


Unless otherwise noted we use a straightforward OpenMP-parallelized C++
implementation of the first order primal-dual method described
in~\cite{pock2011diagonal} to find minimizers of the respective convex
program. As with other proximal methods the algorithm leaves freedom of how
the convex problem is splitted (i.e.\ which dual unknowns are introduced, and
the choice of proximal steps utilized). The employed splitting often has a
significant impact on convergence behavior. In general, we eliminate $\vec z$
from the objective and introduce Lagrange multipliers for each of the
remaining constraints. We also introduce bounds-constrained dual variables for
terms as
\begin{align}
  &\left[ \gamma^{kl} \left(Y_{s,h \to s}^{ki} - Y_{s,h \to t}^{k,i+\delta^{kl}} \right) \right]_+
  = \left[ \gamma^{kl} \left(\sum_{j=0}^{i-1} y_{s,h \to s}^{kj} - \sum_{j=0}^{i-1} y_{s,h \to t}^{k,j+\delta^{kl}} \right) \right]_+ \nonumber \\
  = &\max_{p \in [0, \gamma^{kl}]} p \left( \sum_{j=0}^{i-1} y_{s,h \to s}^{kj} - \sum_{j=0}^{i-1} y_{s,h \to t}^{k,j+\delta^{kl}} \right).
  \nonumber
\end{align}
A naive implementation of the respective primal-dual update steps has a
$O(L^2)$ time complexity (per edge in the graph), and we use appropriate
running sums to preserve the $O(L)$ complexity.

\subsection{Early Stopping of Message Passing Methods}

As first experiment we verify the claim that early stopping can occur
frequently in dual block coordinate methods for MAP inference such as
MPLP~\cite{globerson2007fixing}. We follow the setup of a similar experiment
described in~\cite{schwing2012globally}, but replace the 3-label spin glass
model considered there by problem instances with many labels and piecewise
linear smoothness priors. Our setup is as follows: the domain is a $20 \times
20$ regular grid with the standard 4-connected neighborhood stucture, and the
state space contains 20 labels. The unary potentials are sampled randomly from
a uniform distribution $\theta_{s}^i \sim \mathcal{U}(0,2)$, and the pairwise
potentials are truncated linear ones with
\begin{equation}
  \vartheta_{st}^{h} = \alpha_{st} \min \{ |h| , 2 \}, 
\end{equation}
and $\alpha_{st} \sim \mathcal{U}(0,1)$ are used. We solve 300 random
instances using MPLP and compare the energy to the globally optimal one
obtained by minimizing $E_{\text{LP-MRF}}$ using the primal-dual algorithm. In
about $30\%$ of the problem instances MPLP stops early with an energy
difference of more than 0.001 and in about $26\%$ with more than 0.01 (see
also Fig.~\ref{fig:earlystopping}).

\begin{figure}[htb]
  \centering
  \includegraphics[width=0.7\linewidth]{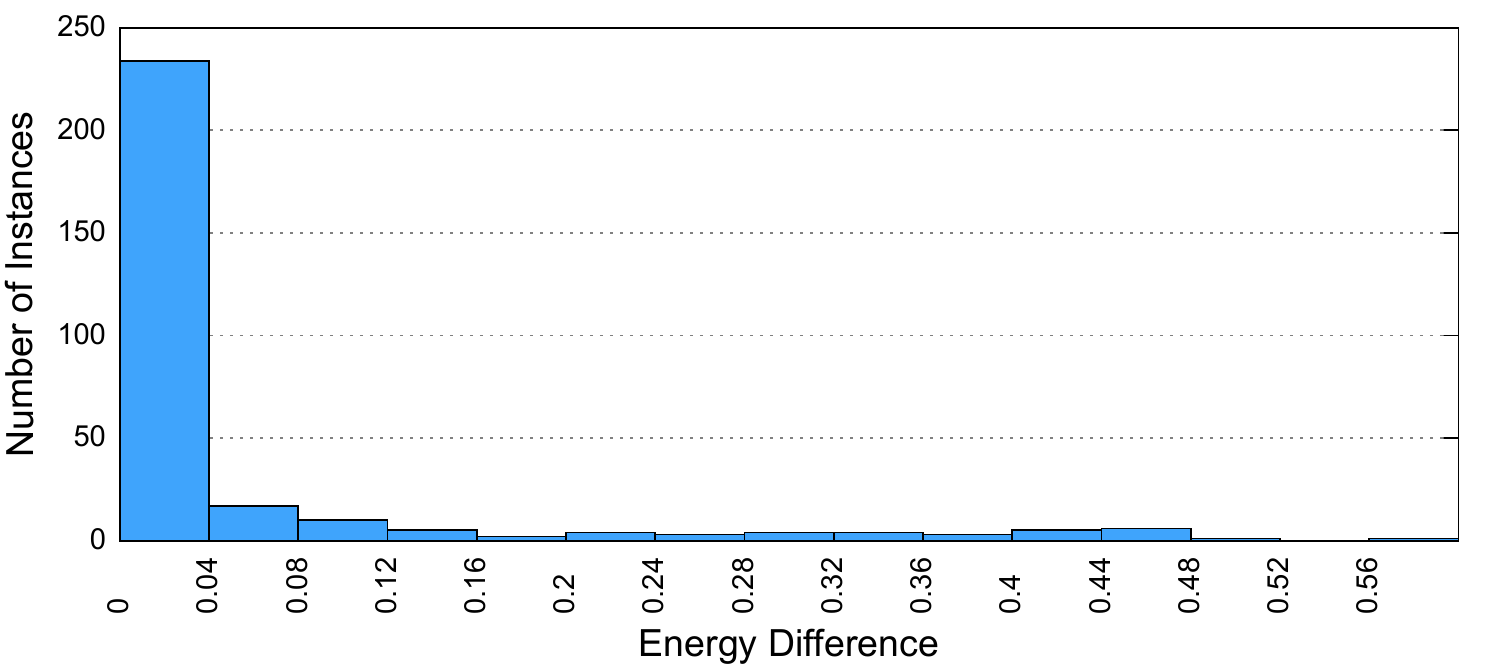}
  \caption{Histogram of optimality gaps for solutions returned by MPLP. }
  \label{fig:earlystopping}
\end{figure}

\subsection{Variable Smoothing}

Recent developments in accelerated gradient methods
(e.g.~\cite{bot2012variable}) appear to be very appealing in order to optimize
the non-smooth dual program of e.g.\ $E_{\text{LP-MRF}}$
(Eq.~\ref{eq:lpmrf}). The method proposed in~\cite{bot2012variable} guarantees
an $O(\ln(T)/T)$ convergence rate, where $T$ is the iteration count, but
requires to set a tuning parameter. In Fig.~\ref{fig:var_smoothing} it is
illustrated that this parameter has to be chosen carefully to achieve
competitive performance. Our conclusion is that such variable smoothing
methods cannot (yet) replace compact reformulations as proposed in the
previous sections.

\begin{figure}
  \centering
  \includegraphics[width=0.7\linewidth]{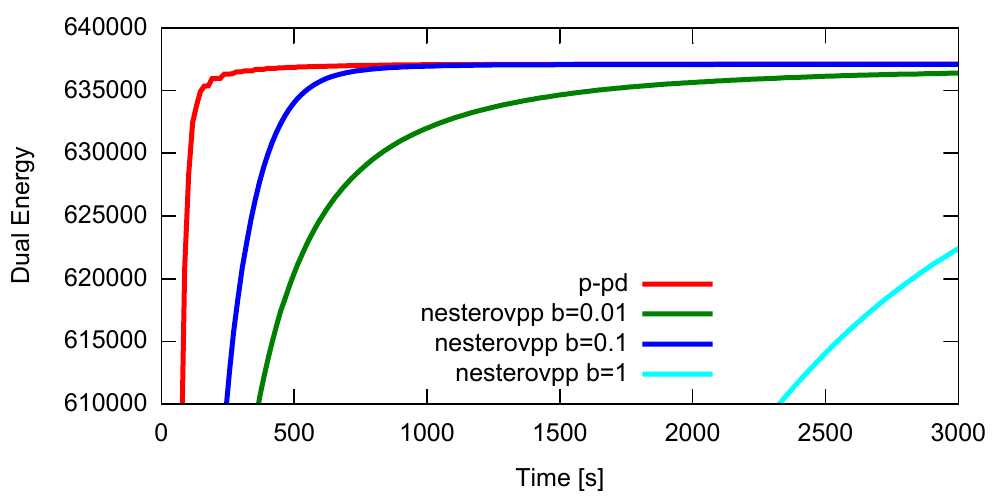}
  \caption{Energy evolution for a 32 label denoising problem utilizing compact
    potential with three pieces. The preconditioned primal dual (p-pd)
    algorithm~\cite{pock2011diagonal} outperforms the variable smoothing
    (nesterovpp) algorithm~\cite{bot2012variable} with different tuning
    parameters $b$}
  \label{fig:var_smoothing}
\end{figure}

\subsection{Convergence Speed and Memory Consumption}

We use a more realistic application to compare memory requirements and the
evolution of energies. While a compact reformulation such as
$E_{\text{min-}L^1}$ has a smaller problem size than $E_{\text{LP-MRF}}$, it
is not clear whether the more complicated problem structure may lead to slower
convergence. We chose a simple image denoising application for this
demonstration. We use a piecewise linear approximation depicted in
Fig.~\ref{subfig:denoisPairwise} of the gradient statistic of natural
images~\cite{huang1999statistics}. The unary potential (shown in
Fig.~\ref{subfig:denoisUnary}) is induced directly by our image corruption
procedure, which is as follows: a random set containing five percent of the
pixels are considered as outliers and their clean intensity values are
replaced by a uniform random value from $[0,255]$. For the remaining inlier
pixels we add Gaussian noise drawn from $\mathcal{N}(0,10)$ to their
respective clean intensities. Thus, the data fidelity term is given by
\begin{align}
  D(\vec u) = \sum_s -\lambda \log\left( \frac{5}{100} + \frac{95}{100} \phi(u_s - g_s; 0, 10) \right), \nonumber
\end{align}
where $\phi(x; \mu, \sigma) = \frac{1}{\sqrt{2\pi\sigma^2}} \exp(-\frac{(x -
  \mu)^2}{2\sigma^2})$ is the density function of the Normal distribution. We
set $\lambda = 1$, and the utilized regularizer
(Fig.~\ref{subfig:denoisPairwise}) is
\begin{align}
  R(\vec u) = \sum_{s \sim t} \min_{k\in \{0,1,2\}} \left\{ \alpha^k |u_s - u_t| + \beta^k \right\} \nonumber
\end{align}
with $(\alpha^0, \beta^0) = (24, 0)$, $(\alpha^1, \beta^1) = (8, 1)$, and
$(\alpha^2, \beta^2) = (3.2, 2)$. The observed $400 \times 300$ noisy image is
illustrated in Fig.~\ref{subfig:denoisInput}, and the recovered image can be
seen in Fig.~\ref{subfig:denoisResult}. The solution image is determined by
extracting the 1/2-isolevel of the superlevel function $X_s^i =
\sum_{j=0}^{i-1} x_s^j$.

We discretize the continuous state space $[0,255]$ into 64 labels. The memory
used to minimize $E_{\text{LP-MRF}}$ for this problem is almost 8GB, and
optimizing $E_{\text{min-}L^1}$ ($E_{\text{min-}L^1}^{\text{isotr.}}$,
respectively) requires about 1.1GB memory. Hence, the latter formulations will
fit in graphics memory and can therefore leverage GPUs for acceleration. The
evolution of the objective $D(\vec u) + R(\vec u)$ is displayed in
Fig.~\ref{fig:denoising_energies}. All primal and dual unknowns are
initialized with 0, which may explain the initial increase in the objective
value. The compact reformulations
$E_{\text{min-}L^1}$/$E_{\text{min-}L^1}^{\text{isotr.}}$ have a clear
advantage over the exhaustive model $E_{\text{LP-MRF}}$.

\begin{figure}
\hfill
\subfigure[Unary potential]{\label{subfig:denoisUnary}\includegraphics[width=0.45\linewidth]{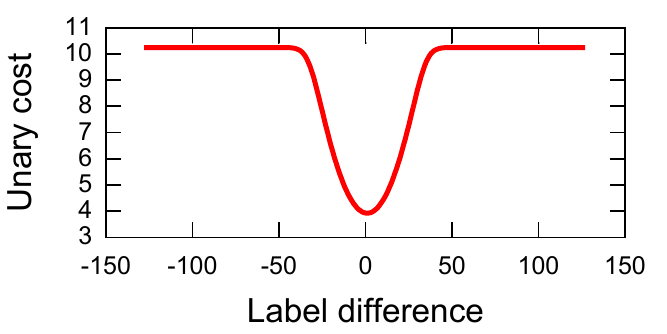}} \hfill
\subfigure[Pairwise potential]{\label{subfig:denoisPairwise}\includegraphics[width=0.45\linewidth]{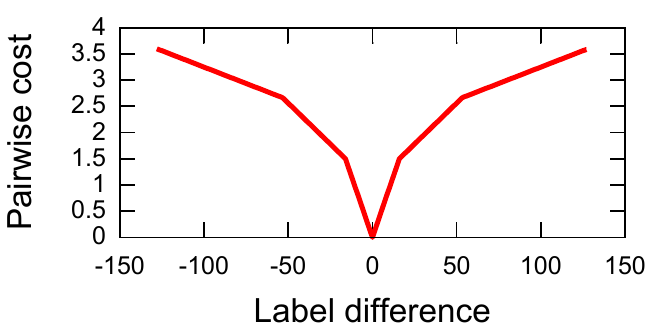}} \\
\subfigure[Noisy input]{\label{subfig:denoisInput}\includegraphics[width=0.45\linewidth]{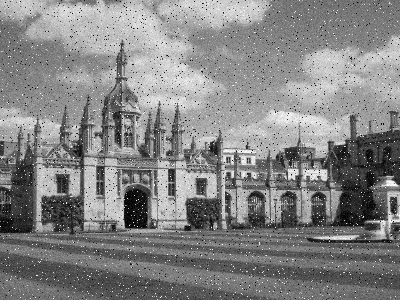}} \hfill
\subfigure[Denoised image]{\label{subfig:denoisResult}\includegraphics[width=0.45\linewidth]{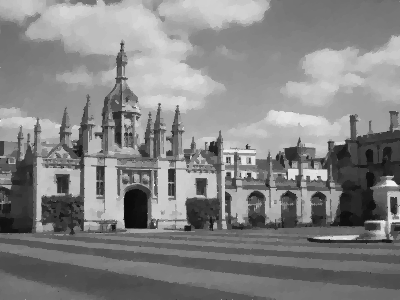}}
\hfill
\caption{Image denoising using 64 labels and a compact piece-wise linear potential.}
\label{fig:experiment}
\end{figure}

\begin{figure}[htb]
  \centering
  \includegraphics[width=0.7\linewidth]{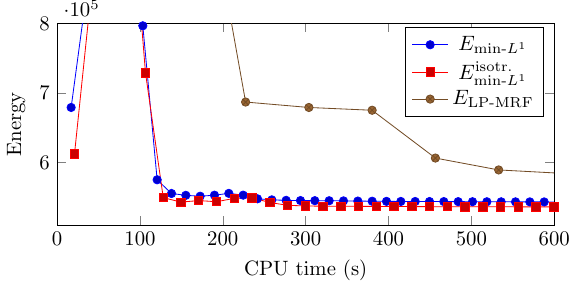}
  \caption{Energy evolution of $E_{\text{LP-MRF}}$ and compact reformulations. }
  \label{fig:denoising_energies}
\end{figure}

\subsection{Comparison with A Continuously-inspired Formulation}

In~\cite{pock2010global} a continuous approach is described that addresses
labeling problems with convex smoothness priors (in terms of the spatial
gradient of the assigned label function) but arbitrary data fidelity term. It
is further shown that in the continuum a global solution can be obtained by
thresholding a minimizer of an underlying convex relaxation. This result does
in general not hold after discretizing the continuous functional. An
interesting example of a convex smoothness prior that is considered
in~\cite{pock2010global} is the Lipschitz prior, $\vartheta^h = \imath\{ |h|
\le \eta\}$ for some $\eta \ge 0$. This regularizer enforces bounded label
differences for adjacent nodes (pixels) and is very compact to represent in
the framework of Section~\ref{sec:conv},
\begin{align}
  f_{st}^{\text{Lipschitz}}(y_s, y_t | x_s, x_t) = \imath\left\{ Y_s^i \le Y_t^{i-L\eta},\, Y_s^i \ge Y_t^{i+L\eta} \right\}
  \nonumber
\end{align}
together with $y_s \in \Delta^L$ and $y_t \in \Delta^L$. $L$ will be 32 in the
following. Note that the energy formulation in~\cite{pock2010global} also
requires only $O(L)$ variables and constraints per edge in the grid.  In order
to have a ground truth result available for better comparison, we use a
(convex) quadratic data term, $(u_s-g_s)^2$, which implies that an optimal
labeling can be easily obtained (see Figs.~\ref{fig:lipschitz}(a,d), we choose
$\eta = 1/16$). Figs.~\ref{fig:lipschitz}(b,e) depict the result of the
discretized model~\cite{pock2010global} with the label space discretized into
$L=32$ states, and Figs.~\ref{fig:lipschitz}(c,f) display the result of
Eq.~\ref{eq:pmin_full_iso} specialized to $f_{st}^{\text{Lipschitz}}$. In
terms of the PSNR our result is much closer to the true minimizer, and further
preserves more image details. This small experiment indicates that often care
is required when working with discretized problems of continuous labeling
functionals.

\begin{figure*}[htb]
  \centering
  \subfigure[True solution]{\includegraphics[width=0.16\linewidth]{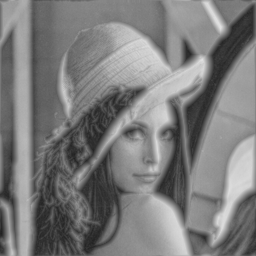}}
  \subfigure[PSNR: 36.03]{\includegraphics[width=0.16\linewidth]{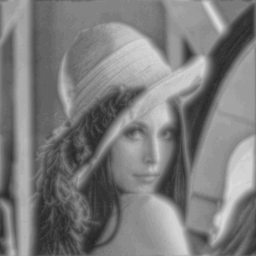}}
  \subfigure[PSNR: 37.56]{\includegraphics[width=0.16\linewidth]{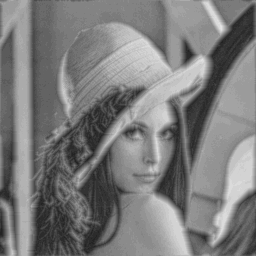}}
  \subfigure[Detail of (a)]{\includegraphics[width=0.16\linewidth]{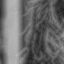}}
  \subfigure[Detail of (b)]{\includegraphics[width=0.16\linewidth]{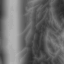}}
  \subfigure[Detail of (c)]{\includegraphics[width=0.16\linewidth]{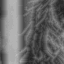}}
  \caption{Ground truth result (a) and solutions returned by a continuous
    formulation (b) and our reformulation (c) for a labeling task with
    Lipschitz prior.  (d--f) depict zoomed in region in the lower left
    corner. }
  \label{fig:lipschitz}
\end{figure*}

\section{Conclusion}

We show that pairwise potentials, that can be written as piecewise linear
functions in terms of the respective label difference, allow compact
reformulations of the LP relaxation for MAP inference. These reformulations do
not weaken the relaxation or modify the returned minimizer. The resulting
savings in memory consumption can be very significant (e.g.\ one order of
magnitude) for many-label problems. The construction also extends to
formulations aiming to reduce the grid bias often used in image processing.

Future work will address the applicability of the underlying techniques
developed in this manuscript to general piecewise linear potentials
$\theta_{st}^{ij}$ (not only those that can be written as $\theta_{st}^{ij} =
\vartheta_{st}^{j-i}$), and to higher order potentials beyond pairwise ones.

{\small
\bibliographystyle{plain}
\bibliography{references}

\begin{thebibliography}{10}

\bibitem{bot2012variable}
Radu~Ioan Bot and Christopher Hendrich.
\newblock A variable smoothing algorithm for solving convex optimization
  problems.
\newblock {\em arXiv preprint arXiv:1207.3254}, 2012.

\bibitem{boykov98markov}
Y.~Boykov, O.~Veksler, and R.~Zabih.
\newblock Markov random fields with efficient approximations.
\newblock In {\em IEEE Conference on Computer Vision and Pattern Recognition
  (CVPR)}, pages 648--655, 1998.

\bibitem{chambolle2012convex}
Antonin Chambolle, Daniel Cremers, and Thomas Pock.
\newblock A convex approach to minimal partitions.
\newblock {\em SIAM Journal on Imaging Sciences}, 5(4):1113--1158, 2012.

\bibitem{chekuri2004linear}
Chandra Chekuri, Sanjeev Khanna, Joseph Naor, and Leonid Zosin.
\newblock A linear programming formulation and approximation algorithms for the
  metric labeling problem.
\newblock {\em SIAM Journal on Discrete Mathematics}, 18(3):608--625, 2004.

\bibitem{felzenszwalb2006efficient}
Pedro~F Felzenszwalb and Daniel~P Huttenlocher.
\newblock Efficient belief propagation for early vision.
\newblock {\em Int. Journal of Computer Vision}, 70(1):41--54, 2006.

\bibitem{globerson2007fixing}
Amir Globerson and Tommi Jaakkola.
\newblock Fixing max-product: Convergent message passing algorithms for map
  lp-relaxations.
\newblock {\em Advances in neural information processing systems (NIPS)},
  21(1.6), 2007.

\bibitem{grady2010discrete}
Leo~John Grady and Jonathan~R Polimeni.
\newblock {\em Discrete calculus: Applied analysis on graphs for computational
  science}.
\newblock Springer, 2010.

\bibitem{hazan2010normproduct}
T.~Hazan and A.~Shashua.
\newblock Norm-prodcut belief propagtion: Primal-dual message-passing for
  lp-relaxation and approximate-inference.
\newblock {\em IEEE Trans. on Information Theory}, 56(12):6294--6316, 2010.

\bibitem{huang1999statistics}
Jinggang Huang and David Mumford.
\newblock Statistics of natural images and models.
\newblock In {\em IEEE Computer Society Conference on Computer Vision and
  Pattern Recognition (CVPR)}, volume~1. IEEE, 1999.

\bibitem{ishikawa2003exact}
Hiroshi Ishikawa.
\newblock Exact optimization for markov random fields with convex priors.
\newblock {\em IEEE Transactions on Pattern Analysis and Machine Intelligence
  (TPAMI)}, 25(10):1333--1336, 2003.

\bibitem{ishikawa1998segmentation}
Hiroshi Ishikawa and Davi Geiger.
\newblock Segmentation by grouping junctions.
\newblock In {\em IEEE Conference on Computer Vision and Pattern Recognition},
  pages 125--131. IEEE, 1998.

\bibitem{kolmogorov2006convergent}
V.~Kolmogorov.
\newblock Convergent tree-reweighted message passing for energy minimization.
\newblock {\em IEEE Transactions on Pattern Analysis and Machine Intelligence
  (PAMI)}, 28(10):1568–--1583, 2006.

\bibitem{kolmogorob2005optimality}
V.~Kolmogorov and M.~Wainwright.
\newblock On the optimality of tree-reweighted max-product message-passing.
\newblock In {\em Proc.\ Uncertainty in Artificial Intelligence (UAI)}, 2005.

\bibitem{kovalesky75msd}
V.A. Kovalevsky and V.K. Koval.
\newblock A diffusion algorithm for decreasing energy of max-sum labeling
  problem.
\newblock {Glushkov} Inst.\ of Cybernetics, Kiev, USSR, 1975.

\bibitem{lemire2006streaming}
Daniel Lemire.
\newblock Streaming maximum-minimum filter using no more than three comparisons
  per element.
\newblock {\em Nordic J. Computing}, 13(4):328--339, 2006.

\bibitem{pock2010global}
T.~Pock, D.~Cremers, H.~Bischof, and A.~Chambolle.
\newblock Global solutions of variational models with convex regularization.
\newblock {\em SIAM Journal on Imaging Sciences}, 3(4):1122--1145, 2010.

\bibitem{pock2011diagonal}
Thomas Pock and Antonin Chambolle.
\newblock Diagonal preconditioning for first order primal-dual algorithms in
  convex optimization.
\newblock In {\em IEEE International Conference on Computer Vision (ICCV)},
  pages 1762--1769. IEEE, 2011.

\bibitem{schlesinger2007permuted}
D.~Schlesinger.
\newblock Exact solution of permuted submodular {MinSum} problems.
\newblock In {\em Proc.\ Energy Minimization Methods in Computer Vision and
  Pattern Recognition (EMMCVPR)}, pages 28--38, 2007.

\bibitem{schwing2012globally}
Alex Schwing, Tamir Hazan, Marc Pollefeys, and Raquel Urtasun.
\newblock Globally convergent dual map lp relaxation solvers using
  fenchel-young margins.
\newblock In {\em Advances in Neural Information Processing Systems (NIPS)},
  pages 2393--2401, 2012.

\bibitem{sontag2011introduction}
D.~Sontag, A.~Globerson, and T.~Jaakkola.
\newblock {\em Optimization for Machine Learning}, chapter Introduction to Dual
  Decomposition for Inference.
\newblock MIT Press, 2011.

\bibitem{wainwright2008graphical}
M.~J. Wainwright and M.~I. Jordan.
\newblock Graphical models, exponential families, and variational inference.
\newblock {\em Found. Trends Mach. Learn.}, 1:1--305, 2008.

\bibitem{werner2007maxsum_review}
T.~Werner.
\newblock A linear programming approach to max-sum problem: A review.
\newblock {\em IEEE Transactions on Pattern Analysis and Machine Intelligence
  (PAMI)}, 29(7), 2007.

\bibitem{yuan2011running}
Hao Yuan and Mikhail~J Atallah.
\newblock Running max/min filters using $1+o(1)$ comparisons per sample.
\newblock {\em Pattern Analysis and Machine Intelligence, IEEE Transactions
  on}, 33(12):2544--2548, 2011.

\bibitem{zach2008labeling}
C.~Zach, D.~Gallup, J.-M. Frahm, and M.~Niethammer.
\newblock Fast global labeling for real-time stereo using multiple plane
  sweeps.
\newblock In {\em Proc.\ VMV}, 2008.

\bibitem{zach2012multilabel}
C.~Zach, C.~H\"ane, and M.~Pollefeys.
\newblock What is optimized in tight convex relaxations for multi-label
  problems?
\newblock In {\em IEEE Conference on Computer Vision and Pattern Recognition
  (CVPR)}, 2012.

\bibitem{zach2013connecting}
C.~Zach, C.~H\"ane, and M.~Pollefeys.
\newblock What is optimized in convex relaxations for multi-label problems:
  Connecting discrete and continuously-inspired {MAP} inference.
\newblock {\em IEEE Transactions on Pattern Analysis and Machine Intelligence
  (PAMI)}, 2013.
\newblock accepted for publication.

\bibitem{zach2012dcmrf}
Christopher Zach and Pushmeet Kohli.
\newblock A convex discrete-continuous approach for markov random fields.
\newblock In {\em European Conference on Computer Vision (ECCV)}, pages
  386--399. Springer Berlin Heidelberg, 2012.

\end{thebibliography}
}

\end{document}